\documentclass[letterpaper, 10 pt, conference]{ieeeconf}  

\IEEEoverridecommandlockouts                              

\usepackage{amsmath}
\usepackage{hyperref}
\usepackage{amsfonts}
\usepackage{color}
\usepackage{array}
\usepackage{graphicx}
\graphicspath{ {./figs/} }
\usepackage{mathtools}
\usepackage{amssymb}
\usepackage{cite}
\usepackage{floatrow}

\newtheorem{theorem}{Theorem}
\newtheorem{problem}{Problem}
\newtheorem{definition}{Definition}

\newtheorem{lemma}{Lemma}

\newcommand{\task}{\Pi}

\newcommand{\traj}{\xi}
\newcommand{\state}{x}
\newcommand{\statespace}{\mathcal{X}}
\newcommand{\safeset}{\mathcal{S}}
\newcommand{\unsafeset}{\mathcal{A}}

\newcommand{\numsafe}{N_s}
\newcommand{\numunsafe}{N_{\neg s}}
\newcommand{\control}{u}
\newcommand{\controlset}{\mathcal{U}}

\newcommand{\trajxu}{\traj_{xu}}
\newcommand{\trajx}{\traj_\state}
\newcommand{\traju}{\traj_\control}

\newcommand{\constraintspace}{\mathcal{C}}

\newcommand{\guarunsafe}{\mathcal{G}_{\neg s}}
\newcommand{\guarsafe}{\mathcal{G}_s}
\newcommand{\feas}{\mathcal{F}}
\newcommand{\cstate}{p}
\newcommand{\numbox}{N^*}

\newcommand{\demj}{\traj_{j}^\textrm{loc}}
\newcommand{\stat}{\textrm{stat}}
\newcommand{\comp}{\textrm{comp}}
\newcommand{\thetac}{\gamma}
\newcommand{\Thetac}{\Gamma}

\newcommand{\eq}{\textrm{eq}}
\newcommand{\ineq}{\textrm{ineq}}
\newtheorem{rem}{Remark}

\title{\LARGE \bf
Learning Constraints from Locally-Optimal Demonstrations under Cost Function Uncertainty
}

\author{Glen Chou, Necmiye Ozay, and Dmitry Berenson$^{1}$
\thanks{$^{1}$Electrical Engineering and Computer Science,
        University of Michigan, Ann Arbor, MI,
        {\tt\small \{gchou, necmiye, dmitryb\}@umich.edu}}%
}

\begin{document}

\maketitle
\thispagestyle{empty}
\pagestyle{empty}

\begin{abstract}
We present an algorithm for learning parametric constraints from locally-optimal demonstrations, where the cost function being optimized is uncertain to the learner. Our method uses the Karush-Kuhn-Tucker (KKT) optimality conditions of the demonstrations within a mixed integer linear program (MILP) to learn constraints which are consistent with the local optimality of the demonstrations, by either using a known constraint parameterization or by incrementally growing a parameterization that is consistent with the demonstrations. We provide theoretical guarantees on the conservativeness of the recovered safe/unsafe sets and analyze the limits of constraint learnability when using locally-optimal demonstrations. We evaluate our method on high-dimensional constraints and systems by learning constraints for 7-DOF arm and quadrotor examples, show that it outperforms competing constraint-learning approaches, and can be effectively used to plan new constraint-satisfying trajectories in the environment.
\end{abstract}

\vspace{-3pt}
\section{Introduction}
\vspace{-2pt}

Learning from demonstration has largely focused on learning cost and reward functions, through the frameworks of inverse optimal control and inverse reinforcement learning (IOC/IRL) \cite{irl_1, irl_2, lfd3, ng_irl}, which replicate the behavior of an expert demonstrator when optimized. However, real-world planning problems, such as navigating a quadrotor in an urban environment, also require that the system obey hard constraints, that is, system trajectories must remain in a set of safe (constraint-satisfying) states. As constraints enforce safety more strictly than cost function penalties, which may ``soften" a constraint and allow violations, they are better suited for planning in safety-critical situations. Furthermore, while different tasks may require different cost functions, oftentimes safety constraints are shared across tasks, and identifying them can help the robot generalize.

Initial work in \cite{wafr} and \cite{corl} has taken steps towards identifying constraints from approximately globally-optimal expert demonstrations, assuming that the demonstrator's cost function is known exactly. However, as humans are not always experts at performing a task, requiring them to provide demonstrations which are nearly globally-optimal can be unreasonable. Furthermore, it is rare for the cost function being optimized to be known exactly by the learner.
To address these shortcomings, we consider the problem of learning parametric constraints shared across tasks from approximately locally-optimal demonstrations under parametric cost function uncertainty. Our method is based on the insight that locally-optimal, constraint-satisfying demonstrations satisfy the Karush-Kuhn-Tucker (KKT) optimality conditions, which are first-order necessary conditions for local optimality of a solution to a constrained discrete-time optimal control problem. We solve a mixed integer linear program (MILP) to recover constraint and cost function parameters which make the demonstrations locally-optimal. We make the following specific contributions in this paper:
\vspace{-1pt}
\begin{itemize}
	\item We develop a novel algorithm for learning parametric, potentially non-convex constraints from approximately locally-optimal demonstrations, where the parameterization can either be provided or grown incrementally to be consistent with the data. The method can extract volumes of safe/unsafe states (states which satisfy/do not satisfy the constraints) for future guaranteed safe planning and enable planners to query states for safety.\vspace{-1pt}
	\item Our method can learn constraints despite uncertainty in the cost function and can also recover a cost function jointly with the constraint.\vspace{-1pt}
	\item Under mild assumptions, we prove that our method recovers guaranteed conservativeness estimates (that is, inner approximations) of the true safe/unsafe sets, and analyze the learnability of a constraint from locally-optimal compared to globally-optimal demonstrations.\vspace{-1pt}
	\item We evaluate our method on difficult constraint learning problems in high-dimensional constraint spaces (23 dimensions) on systems with complex nonlinear dynamics and demonstrate that our method outperforms previous methods for parametric constraint inference \cite{wafr, corl}.\vspace{-1pt}
\end{itemize}

\vspace{-2pt}
\section{Related Work}
\vspace{-1pt}

Previous work in IOC \cite{boyd, toussaint, pontryagin} has used the KKT conditions to recover a cost function from demonstrations, assuming that the constraints are known. In contrast, our method explicitly learns the constraints. The risk-sensitive IRL approach in \cite{sumeet} also uses the KKT conditions, and is complementary to our work, which learns hard constraints. Perhaps the closest to our work is \cite{melanie}, which aims to recover a cost function and constraint simultaneously using the KKT conditions. However, to avoid non-convexity in the cost/constraint recovery problem, \cite{melanie} restricts their method to recovering convex constraints and do not directly search for constraint parameters, instead enumerating an a-priori fixed, finite set of candidate constraint parameters using a method which holds only for the convex case. In contrast, by formulating our problem as a MILP, our method avoids enumeration, directly searching the full continuous space of possible constraint parameters. It also enables us to learn non-convex, nonlinear constraints while retaining formal guarantees on the conservativeness of the recovered constraint.
 
There also exists prior work in learning geometric state space constraints \cite{vijayakumar}, \cite{shah}, which our method generalizes by learning non-convex constraints not necessarily defined in the state space. Learning local trajectory-based constraints has also been explored in \cite{dmitry, anca, lfdc1,lfdc2,lfdc3,lfdc4} by reasoning over the constraints within a single trajectory or task. These methods complement our approach, which learns a global constraint shared across tasks. In the constraint-learning literature, our work is closest to \cite{wafr} and \cite{corl}, which learn a global shared constraint by sampling unsafe trajectories using globally-optimal demonstrations, a known cost function, and a system simulator. In addition to the drawbacks of global optimality and assuming a known cost as mentioned previously, sampling unsafe trajectories can be difficult for systems with complicated dynamics. By using the KKT conditions, which \textit{implicitly} define the unsafe set instead of explicitly through unsafe trajectories, our method sidesteps both the need for an exact cost function to classify the safety of sampled trajectories as well as any sampling difficulties.

\section{Preliminaries and Problem Setup}

We consider discrete-time nonlinear systems $\state_{t+1} = f(\state_t, \control_t, t)$, $\state\in\statespace$ and $\control\in\controlset$, performing tasks $\task$, which are represented as constrained optimization problems over state/control trajectories $\trajxu\doteq(\trajx,\traju)$:
\begin{problem}[Forward problem / ``task" $\task$]\label{prob:fwd_prob}

\vspace{2pt}
\begin{equation}\label{eq:fwdprob}
	\begin{array}{>{\displaystyle}c >{\displaystyle}l >{\displaystyle}l}
				&\\[-18pt]
		\underset{\trajxu}{\text{minimize}} & \quad c_\task(\trajxu, \thetac) &\\
		\text{subject to} & \quad \phi(\trajxu) \in \safeset(\theta) \subseteq \constraintspace\\
		& \quad \bar\phi(\trajxu) \in \bar\safeset \subseteq \bar\constraintspace\\
		& \quad \phi_\task(\trajxu) \in \safeset_\task \subseteq \constraintspace_\task\\[3pt]
	\end{array}\hspace{-15pt}
\end{equation}
\end{problem}

\noindent where $c_\task(\cdot, \thetac)$ is a potentially non-convex cost function for task $\task$, parameterized by $\thetac \in \Thetac$. In Sec. \ref{sec:method_kkt_base} to \ref{sec:method_volume}, we assume that $\thetac$ is known (through possibly inaccurate prior knowledge) for clarity; we later relax this assumption and discuss how to learn $\gamma$ from the demonstrations. Further, $\phi(\cdot)$ is a known mapping from state-control trajectories to a constraint space $\constraintspace$, elements of which are referred to as \textit{constraint states} $\cstate \in \constraintspace$. Mappings $\bar\phi(\cdot)$ and $\phi_\task(\cdot)$ are known and map to constraint spaces $\bar\constraintspace$ and  $\constraintspace_\task$, containing a known shared safe set $\bar \safeset$ and a known task-dependent safe set $\safeset_\task$, respectively. In this paper, we encode the system dynamics in $\bar\safeset$ and start/goal constraints in $\safeset_\task$. Grouping the constraints of Problem \ref{prob:fwd_prob} as equality/inequality (eq/ineq) constraints and known/unknown ($k/\neg k$) constraints, we can write:
\begin{equation}\label{eq:prob1ineq_eq}
\hspace{-7pt}\begin{array}{>{\displaystyle}l >{\displaystyle}l >{\displaystyle}l}
		 h_{i,k}(\trajxu) = 0, i=1,...,N_k^\eq &\hspace{-5pt} \Leftrightarrow \mathbf{h}_k(\trajxu) = \mathbf{0}\\
		 g_{i,k}(\trajxu) \le 0, i=1,...,N_k^\ineq &\hspace{-5pt} \Leftrightarrow \mathbf{g}_{k}(\trajxu) \le \mathbf{0}\\
		 g_{i,\neg k}(\trajxu, \theta) \le 0, i=1,...,N_{\neg k}^\ineq &\hspace{-5pt} \Leftrightarrow \mathbf{g}_{\neg k}(\trajxu, \theta) \le \mathbf{0}\\
	\end{array}\hspace{-7pt}
\end{equation}

\noindent where $\mathbf{h}_k(\trajxu) \in \mathbb{R}^{N_k^\eq}$, $\mathbf{g}_{k}(\trajxu) \in \mathbb{R}^{N_k^\ineq}$, and $\mathbf{g}_{\neg k}(\trajxu, \theta) \in \mathbb{R}^{N_{\neg k}^\ineq}$. Note that unknown equality constraints $\mathbf{h}_{\neg k}(\trajxu, \theta) = 0$ can be written equivalently as $\mathbf{h}_{\neg k}(\trajxu, \theta) \le 0, -\mathbf{h}_{\neg k}(\trajxu, \theta) \le 0$. As shorthand, let $g(\cstate,\theta) \doteq \max_{i\in\{ 1, \ldots, N_{\neg k}^{\ineq}\}}\big(g_{i,\neg k}(\cstate, \theta)\big)$. We now define
\vspace{-2pt}
\begin{align}
	\safeset(\theta) &\doteq \{\cstate \in \constraintspace\mid g(\cstate, \theta) \le 0\}\\
    \unsafeset(\theta) &\doteq \safeset(\theta)^c = \{\cstate \in \constraintspace\mid g(\cstate, \theta) > 0)\}
\end{align}

\vspace{-0pt}
\noindent as an unknown safe/unsafe set defined by unknown parameter $\theta \in \Theta$, for possibly unknown parameterizations $g_{i, \neg k}(\cdot, \cdot)$. Last, we restrict $\Thetac$ and $\Theta$ to be unions of polytopes.

Intuitively, a trajectory $\trajxu$ is locally-optimal if all trajectories within a neighborhood of $\trajxu$ have cost greater than or equal to $c(\trajxu)$. More precisely, for a trajectory to be locally-optimal, it necessarily satisfies the KKT conditions \cite{cvxbook}. We define a demonstration $\traj^\textrm{loc}$ as a state-control trajectory which we assume approximately solves Problem \ref{prob:fwd_prob} to local optimality, i.e. it satisfies all constraints and is in the neighborhood of a local optimum. 

Our goal is to recover the safe set $\safeset(\theta)$ and unsafe set $\unsafeset(\theta)$, given $\numsafe$ demonstrations $\{\traj_{j}^\textrm{loc} \}_{j=1}^{\numsafe}$, known shared safe set $\bar\safeset$, and task-dependent constraints $\safeset_\task$. As a byproduct, our method can also recover unknown cost parameters $\thetac$. 
\vspace{-3pt} 
\section{Method}
\vspace{-3pt}
We detail our constraint-learning algorithm. First, we formulate the general KKT-based constraint recovery problem (Sec. \ref{sec:method_kkt_base}) and then develop specific optimization problems for the cases where the constraint is defined as a union of offset-parameterized (Sec. \ref{sec:method_unions}) or affinely-parameterized constraints (Sec. \ref{sec:method_linear}). We show how to extract guaranteed safe/unsafe states (Sec. \ref{sec:method_volume}), handle unknown constraint parameterizations (Sec. \ref{sec:method_unknown}), and handle cost function uncertainty (Sec. \ref{sec:method_cost}). In closing, we show how our method can be used within a planner to guarantee safety (Sec. \ref{sec:method_planning}).

\vspace{-3pt}
\subsection{Constraint recovery via the KKT conditions}\label{sec:method_kkt_base}
\vspace{-1pt}
Recall that the KKT conditions are necessary conditions for local optimality of a solution of a constrained optimization problem \cite{cvxbook}. 
For constraints \eqref{eq:prob1ineq_eq} and Lagrange multipliers $\lambda$ and $\nu$, the KKT conditions for the $j$th locally-optimal demonstration $\traj_{j}^\textrm{loc}$, denoted $\textrm{KKT}(\traj_{j}^\textrm{loc})$, are:

\vspace{-12pt}\begin{subequations}\label{eq:kkt}
	\small\begin{align}
	\hspace{-25pt}\textrm{Primal feasibility:}\ \  &\mathbf{h}_{k}(\traj_{j}^\textrm{loc}) = \mathbf{0},\label{eq:kkt_primal1}\\[-2pt]
	\hspace{-25pt} \quad &\mathbf{g}_{k}(\traj_{j}^\textrm{loc}) \le \mathbf{0}, \label{eq:kkt_primal2}\\[-3pt]
	\hspace{-25pt}  &\mathbf{g}_{\neg k}(\traj_{j}^\textrm{loc}, \textcolor{red}{\theta}) \le \mathbf{0}, \label{eq:kkt_primal3}\\[1pt]
	\hline\hspace{-25pt}\textrm{Lagrange mult.}\quad &\textcolor{blue}{\lambda_{i,k}^j} \ge 0, \ \ \hspace{-1.5pt} i = 1,..., N_k^\ineq\ \ \Leftrightarrow \ \textcolor{blue}{\boldsymbol{\lambda}_k^j} \ge \mathbf{0}\label{eq:kkt_lag1}\\[-3pt]
	\textrm{nonnegativity:}\quad &\textcolor{blue}{\lambda_{i,\neg k}^j} \ge 0, \ i = 1,..., N_{\neg k}^\ineq\ \Leftrightarrow\ \textcolor{blue}{\boldsymbol{\lambda}_{\neg k}^j} \ge \mathbf{0}\label{eq:kkt_lag2}\\
	\hline\hspace{-3pt}\textrm{Complementary}\quad &\textcolor{blue}{\boldsymbol{\lambda}_{k}^j}\odot\mathbf{g}_{k}(\traj_{j}^\textrm{loc}) = \mathbf{0}\label{eq:kkt_comp1}\\[-3pt]
	\textrm{slackness:}\quad &\textcolor{blue}{\boldsymbol{\lambda}_{\neg k}^j}\odot\mathbf{g}_{\neg k}(\traj_{j}^\textrm{loc}, \textcolor{red}{\theta}) = \mathbf{0}\label{eq:kkt_comp2}\\
	\hline\notag\\[-12pt]
	\hspace{-25pt}\textrm{Stationarity:}\quad &\nabla_{\trajxu} c_\task(\traj_{j}^\textrm{loc}) + \textcolor{blue}{\boldsymbol{\lambda}_{k}^{j}}^\top \nabla_{\trajxu} \mathbf{g}_{k}(\traj_{j}^\textrm{loc})\notag
	\\[-2pt]&\quad+\textcolor{blue}{\boldsymbol{\lambda}_{\neg k}^j}^{\hspace{-4pt}\top} \nabla_{\trajxu} \mathbf{g}_{\neg k}(\traj_{j}^\textrm{loc}, \textcolor{red}{\theta})\label{eq:kkt_stat}
	\\[-2pt]&\quad+\textcolor{blue}{\boldsymbol{\nu}_{k}^j}^\top \nabla_{\trajxu} \mathbf{h}_{k}(\traj_{j}^\textrm{loc}) = \mathbf{0}\notag
\end{align}
\end{subequations}\vspace{-16pt}

\noindent where $\nabla_{\trajxu}(\cdot)$ takes the gradient with respect to a flattened trajectory $\trajxu$ and $\odot$ denotes elementwise multiplication. For compactness, we vectorize the multipliers $\boldsymbol{\lambda}_{k}^j \in \mathbb{R}^{N_k^\ineq}$, $\boldsymbol{\lambda}_{\neg k}^j \in \mathbb{R}^{N_{\neg k}^\ineq}$, and $\boldsymbol{\nu}_{k}^j \in \mathbb{R}^{N_k^\ineq}$. We drop the $\thetac$ dependency, as the cost is assumed known for now, as well as \eqref{eq:kkt_primal1}-\eqref{eq:kkt_primal2}, as they involve no decision variables. Then, finding a constraint  consistent with the local optimality conditions of the $\numsafe$ demonstrations amounts to finding a constraint parameter $\theta$ which satisfies the KKT conditions for each demonstration. That is, we can solve the following feasibility problem:
\vspace{3pt}
\begin{problem}[KKT inverse, locally-optimal]\label{prob:kkt_opt}
\begin{equation}\vspace{2pt}\label{eq:fwdprob}
	\hspace{-27pt}\begin{array}{>{\displaystyle}c >{\displaystyle}l >{\displaystyle}l}
				&\\[-12pt]
		\text{find} & \theta, \boldsymbol{\lambda}_{k}^j, \boldsymbol{\lambda}_{\neg k}^j,\boldsymbol{\nu}_{k}^j,\ \ j = 1,...,\numsafe\\
		\text{s.t.} & \{\textrm{KKT}(\traj_{j}^\textrm{loc})\}_{j=1}^{\numsafe}\\[1pt]
	\end{array}\hspace{-20pt}
\end{equation}
\end{problem}

Further, to address suboptimality (i.e. approximate local-optimality) in the demonstrations, we can relax the stationarity \eqref{eq:kkt_stat} and complementary slackness constraints \eqref{eq:kkt_comp1}-\eqref{eq:kkt_comp2} and place corresponding penalties into the objective function:

\begin{problem}[KKT inverse, suboptimal]\label{prob:kkt_subopt}

\begin{equation}\label{eq:fwdprob}
	\hspace{-6pt}\begin{array}{>{\displaystyle}c >{\displaystyle}l >{\displaystyle}l}
				&\\[-19pt]
		\underset{\theta, \boldsymbol{\lambda}_{k}^j, \boldsymbol{\lambda}_{\neg k}^j,\boldsymbol{\nu}_{k}^j}{\text{minimize}} & \sum_{j=1}^{\numsafe}\big(\Vert \stat(\demj)\Vert_1 + \Vert \comp(\demj)\Vert_1\big)\\
		\text{subject to} & \eqref{eq:kkt_primal3}-\eqref{eq:kkt_lag2}, \ \forall\demj, \ j = 1,\ldots,\numsafe\\[2pt]
	\end{array}\hspace{-15pt}
\end{equation}
\end{problem}

\noindent where $\stat(\demj)$ denotes the LHS of Eq. \eqref{eq:kkt_stat} and $\comp(\demj)$ denotes the concatenated LHSs of Eqs. \eqref{eq:kkt_comp1} and \eqref{eq:kkt_comp2}.

Denote the projection of the feasible set of Problem \ref{prob:kkt_opt} onto $\Theta$ as $\feas$. We define the set of learned guaranteed safe/unsafe constraint states as $\guarsafe$/$\guarunsafe$, respectively. For Problem \ref{prob:kkt_opt}, a constraint state $\cstate$ is learned guaranteed safe/unsafe if $\cstate$ is marked safe/unsafe for all $\theta \in \feas$. Formally, we have:

\vspace{-4pt}
\hspace{-23pt}\begin{minipage}{.52\linewidth}
	\small\begin{equation}\label{eq:guarsafe}
		\guarsafe\hspace{-3pt} \doteq\hspace{-3pt} \bigcap_{\theta \in \feas} \{ \cstate\ |\ g(\cstate, \theta) \le 0 \}\hspace{-7pt}
	\end{equation}
\end{minipage}\hspace{5pt}\begin{minipage}{.52\linewidth}
	\small\begin{equation}\label{eq:guarunsafe}
		\guarunsafe\hspace{-3pt} \doteq\hspace{-3pt} \bigcap_{\theta \in \feas} \{ \cstate\ |\ g(\cstate, \theta) > 0 \}\hspace{-7pt}
	\end{equation}
\end{minipage}

\vspace{5pt}We now formulate variants of Problem \ref{prob:kkt_opt} which are efficiently solvable for specific constraint parameterizations. For legibility, we describe the method assuming $\constraintspace = \statespace$ and $\phi(\cdot)$ is the identity. Due to the bilinearity between decision variables in Problems \ref{prob:kkt_opt} and \ref{prob:kkt_subopt} for some parameterizations, we describe exact (Sec. \ref{sec:method_unions}) and relaxed (Sec. \ref{sec:method_linear}) formulations for recovering the unknown parameters.
\vspace{-5pt}
\subsection{Unions of offset-parameterized constraints}\label{sec:method_unions}
Consider when Problem \ref{prob:fwd_prob} involves avoiding an unsafe set $\unsafeset(\theta)$ described by the union and intersection of offset-parameterized half-spaces (i.e. $\theta$ does not multiply $\cstate$):
\vspace{-6pt}\begin{equation}\label{eq:unions_of_offsets}
\unsafeset(\theta) = \bigcup_{m=1}^{N_c} \bigcap_{n=1}^{N_c^m} \{ \cstate\mid a_{m,n}^\top \cstate < b_{m,n}(\theta)\}
\end{equation}\vspace{-10pt}

This parameterization can represent any arbitrarily-shaped unsafe set if $N_c$ is sufficiently large (i.e. as a union of polytopes) \cite{tao}, though in practice our method may not be efficient for large $N_c$. We will often use the specific case of unions of axis-aligned $N_c^m$-dimensional hyper-rectangles,
\vspace{-4pt}\begin{equation}\label{eq:rect}
	\unsafeset(\theta) = \bigcup_{m=1}^{N_c} \bigcap_{n=1}^{N_c^m} \{ \cstate\mid \cstate < \overline{\theta}_n^m, -\cstate < -\underline{\theta}_n^m\},
\end{equation}\vspace{-8pt}

\noindent where $\overline{\theta}_n^m/\underline{\theta}_n^m$ are the upper/lower extents of dimension $n$ of box $m$. We now modify the KKT conditions to handle the ``or" constraints in \eqref{eq:unions_of_offsets}. Primal feasibility \eqref{eq:kkt_primal3} changes to

\vspace{-12pt}
\begin{equation}\label{eq:kkt_or}
\hspace{-5pt} \small\forall \cstate \in \demj, \forall m=1,...,N_c, \bigvee_{n=1}^{N_c^m}\Big( a_{m,n}^\top \cstate\ge b_{m,n}(\theta)\Big),\hspace{-3pt}
\end{equation}\vspace{-4pt}

\vspace{-3pt}\noindent which can be implemented using the big-M formulation \cite{bertsimas}:
\vspace{-7pt}\begin{equation}\label{eq:primal_feas_Ab}
	\small \forall \cstate, \forall m,\ \mathbf{A}_{m}^\top \cstate \ge \mathbf{b}_{m}(\theta) - M(\mathbf{1}-\mathbf{z}_{m}^{j,\cstate}),\ \ \sum_{n=1}^{N_c^m} \mathbf{z}_{m}^{j,\cstate}(n) \ge 1,
\end{equation}\vspace{-7pt}

\noindent where $M$ is a large positive number, $\mathbf{A}_m\in \mathbb{R}^{N_c^m\times |\cstate|}$ and $\mathbf{b}_m\in \mathbb{R}^{N_c^m}$ are the vertical concatenation of $a_{m,n}$ and $b_{m,n}$ for all $n$, $\mathbf{z}_{m}^{j,\cstate} \in \{0,1\}^{N_c^m}$ are binary variables encoding that at least one half-space constraint must hold, and $\mathbf{z}_m^{j,\cstate}(n)$ is the $n$th entry of $\mathbf{z}_m^{j,\cstate}$. For demonstration $\demj$ to be locally-optimal, we know that for each $\cstate \in \demj$, the complementary slackness condition, $\lambda_{(m,n),\neg k}^{j,\cstate}(a_{m,n}^\top \cstate - b_{m,n}(\theta)) = 0$, must hold for at least one $n$ and for all $m$ in Eq. \eqref{eq:kkt_or}. Furthermore, in the stationarity condition \eqref{eq:kkt_stat}, $\lambda_{(m,n),\neg k}^{j,\cstate} \nabla_{\cstate} g_{(m,n),\neg k}(\traj_{j}^\textrm{loc}, \theta)$ terms should only be included for $(m,n)$ pairs where the complementary slackness condition is enforced. Thus, we can enforce that $\lambda_{(m,n),\neg k}^{j,\cstate}(a_{m,n}^\top \cstate - b_{m,n}(\theta)) = 0$ holds for all $\cstate \in \demj$, for all $m \in \{1,...,N_c\}$, and for some $n \in \{1,...,N_c^m\}$ by writing:

\vspace{-12pt}\begin{equation}\footnotesize
	\hspace{-5pt}\begin{gathered}\label{eq:bilinear_stationarity}
	\forall \cstate, m, \begin{bmatrix}
		\boldsymbol{\lambda}_{m,\neg k}^{j,\cstate}\\ \ \hspace{-3pt}\mathbf{A}_m^\top \cstate - \mathbf{b}_m(\theta)
	\hspace{-3pt}\end{bmatrix}\hspace{-3pt}\le M\hspace{-3pt}\begin{bmatrix} \mathbf{z}_{m,1}^{j,\cstate}\\\mathbf{z}_{m,2}^{j,\cstate}\end{bmatrix}\hspace{-3pt},\quad\mathbf{z}_{m,1}^{j,\cstate}+\mathbf{z}_{m,2}^{j,\cstate} \le 2-\mathbf{q}_m^{j,\cstate},\\[-1pt]
	\ \ \sum_{n=1}^{N_c^m} \mathbf{q}_{m}^{{j,\cstate}}(n) \ge 1, \quad \mathbf{z}_{m,1}^{j,\cstate},\ \mathbf{z}_{m,2}^{j,\cstate},\ \mathbf{q}_{m}^{j,\cstate} \in \{0, 1\}^{N_c^m}\\[0pt]
\end{gathered}\hspace{-5pt}
\end{equation}\vspace{-8pt}

\noindent together with \eqref{eq:kkt_lag1} and \eqref{eq:kkt_lag2}, where we use a big-M formulation with binary variables $z$ (encoding the complementary slackness condition) and $q$ (encoding if the complementary slackness condition is being enforced). We have denoted $\boldsymbol{\eta}_{m,\neg k}^{j,\cstate} \doteq [\eta_{(m,1),\neg k}^{j,\cstate}, \ldots, \eta_{(m,N_c^m),\neg k}^{j,\cstate}]^\top$ for $\eta \in \{\lambda, z, q\}$. Next, we modify line 2 of constraint \eqref{eq:kkt_stat} to enforce:
\vspace{-4pt}\begin{equation}\label{eq:bilinear}
	\footnotesize \hspace{-5pt}\sum_{m=1}^{N_c} \big[{(\boldsymbol{\lambda}_{m,\neg k}^{j,\cstate}}\odot \mathbf{q}_{m}^{j,\cstate})^\top \nabla_\cstate (\mathbf{b}_n(\theta) - \mathbf{A}_m^\top \cstate)\big] \doteq \sum_{m=1}^{N_c} {\mathbf{q}_{m}^{j,\cstate}}^\top \mathbf{L}_m^{j,\cstate}
\end{equation}
\vspace{-8pt}

\noindent for all $\cstate\in\demj$, where the $(i,n)$-th entry of $\mathbf{L}_m^{j,\cstate}\in \mathbb{R}^{N_c^m \times |\cstate|}$, $\mathbf{L}_m^{j,\cstate}(i,n)$, refers to $\lambda_{(m,n),\neg k}^{j,\cstate} \nabla_{\cstate(i)} (b_{m,n}(\theta) - a_{m,n}^\top \cstate)$. Note that $\lambda$ (continuous variables) and $q$ (binary variables) are bilinear ($\nabla_\cstate (b_{m,n}(\theta) - a_{m,n}^\top \cstate)$ has no decision variables as $\theta$ does not multiply $\cstate$). By assuming bounds $\underline M \le \mathbf{L}_m^{j,\cstate}(i,n) \le \overline M$, this can be reformulated \textbf{exactly} in a linear fashion (i.e. \textit{linearized}) \cite{mccormick} by replacing each bilinear product $\mathbf{q}_m^{j,\cstate}(n)\mathbf{L}_m^{j,\cstate}(i,n)$ in \eqref{eq:bilinear} with slack variables $\mathbf{R}_{m}^{j,\cstate}(i,n)$ and adding constraints (where $\mathbf{\tilde q}_{m}^{j,\cstate}(n)\doteq (1-\mathbf{q}_{m}^{j,\cstate}(n)$ for short):

\vspace{-9pt}\begin{equation}
\footnotesize\hspace{-8pt}\begin{gathered}\label{eq:bilinear_slack}
	\textrm{min}(0, \underline M) \le \mathbf{R}_{m}^{j,\cstate}(i,n) \le \overline M\\
	\underline M \mathbf{q}_{m}^{j,\cstate}(n) \le \mathbf{R}_{m}^{j,\cstate}(i,n) \le \overline M \mathbf{q}_{m}^{j,\cstate}(n) \\
	\mathbf{L}_{m}^{j,\cstate}(i,n) - \mathbf{\tilde q}_{m}^{j,\cstate}(n)\overline M \le \mathbf{R}_{m}^{{j,\cstate}}(i,n) \le \mathbf{L}_{m}^{j,\cstate}(i,n) - \mathbf{\tilde q}_{m}^{j,\cstate}(n)\underline M \\
	\mathbf{R}_{m}^{j,\cstate}(i,n) \le \mathbf{L}_m^{j,\cstate}(i,n) + \mathbf{\tilde q}_{m}^{j,\cstate}(n)\overline M
\end{gathered}\hspace{-5pt}
\end{equation}\vspace{-5pt}

\noindent Finally, let $\mathbf{R}_{m}^j$/$\mathbf{L}_{m}^j$ be the horizontal concatenation of $\mathbf{R}_{m}^{j,\cstate}$/$\mathbf{L}_{m}^{j,\cstate}$, for all $\cstate \in \demj$. We can now pose the full problem:

\vspace{3pt}
\begin{problem}[KKT inverse, unions]\label{prob:kkt_union}
	\begin{equation}\small
	\hspace{-15pt}\begin{array}{>{\displaystyle}c >{\displaystyle}l >{\displaystyle}l}
				&\\[-13pt]
		\text{find} & \theta, \boldsymbol{\lambda}_{k}^{j,\cstate}, \boldsymbol{\lambda}_{\neg k}^{j,\cstate},\boldsymbol{\nu}_{k}^j,\mathbf{R}_{m}^j, \mathbf{L}_{m}^j, \mathbf{q}_m^{j,\cstate}, \mathbf{z}_{m,1}^{j,\cstate}, \mathbf{z}_{m,2}^{j,\cstate}, \\\
		&\quad \forall\cstate \in \demj, m = 1,...,N_c, j = 1,...,\numsafe\\
		\text{s.t.} & \texttt{-------- Primal feasibility --------} \\[-2pt]
		& \textrm{Equation}\ \eqref{eq:primal_feas_Ab},\ j=1,...,N_s\\
		& \texttt{--- Lagrange mult. nonnegativity ---} \\[-1pt]
		& \boldsymbol{\lambda}_{k}^{j,\cstate} \ge \mathbf{0},\ \forall \cstate \in \demj,\ j = 1,\ldots,\numsafe \\
		& \boldsymbol{\lambda}_{m,\neg k}^{j,\cstate} \ge \mathbf{0},\ \forall \cstate \in \demj, \ m = 1,...,N_c, \ j = 1,...,\numsafe \\[-2pt]
		& \texttt{------ Complementary slackness -----} \\
		& \textrm{Equation}\ \eqref{eq:bilinear_stationarity},\ j = 1, \ldots, N_s \\[-1pt]
		& \texttt{----------- Stationarity -----------} \\[-2pt]
		&\nabla_{\trajxu} c_\task(\traj_{j}^\textrm{loc}) +{\boldsymbol{\lambda}_{\neg k}^j}^{\hspace{-6pt}\top} \nabla_{\trajxu} \mathbf{g}_{k}(\traj_{j}^\textrm{loc})+\sum_{m=1}^{N_c} \mathbf{1}_{N_c^m}^\top\mathbf{R}_{m}^{j}\\[0pt]
		&\quad+{\boldsymbol{\nu}_{\neg k}^j}^{\hspace{-6pt}\top} \nabla_{\trajxu} \mathbf{h}_{k}(\traj_{j}^\textrm{loc}) = 0,\ j=1,...,N_s\\
		&\textrm{Equation } \eqref{eq:bilinear_slack},\ \forall \cstate\in\demj, m = 1,...,N_c,\\
		&\quad n = 1,...,N_c^m, j=1,...,N_s
	\end{array}\hspace{-19pt}
\end{equation}
\end{problem}

\vspace{1pt}
\subsection{Extraction of safe and unsafe states}\label{sec:method_volume}

Before moving onto affine parameterizations, we first detail how to check guaranteed safeness/unsafeness (as defined in \eqref{eq:guarsafe}-\eqref{eq:guarunsafe}). One can check if a constraint state $\cstate \in \guarsafe$ or $\cstate \in \guarunsafe$ by adding a constraint $g(\cstate, \theta) > 0$ or $g(\cstate, \theta) \le 0$ to Problem \ref{prob:kkt_opt} and checking feasibility of the resulting program:

\vspace{2pt}
\begin{problem}[Query if $\cstate$ is guaranteed safe \textbf{OR} unsafe]\label{prob:query}
	\begin{equation}\vspace{1pt}\label{eq:fwdprob}
	\hspace{-27pt}\begin{array}{>{\displaystyle}c >{\displaystyle}l >{\displaystyle}l}
		&\\[-16pt]
		\text{find} & \theta, \boldsymbol{\lambda}_{k}^j, \boldsymbol{\lambda}_{\neg k}^j,\boldsymbol{\nu}_{k}^j,\ \ j = 1,...,\numsafe\\
		\text{s.t.} & \{\textrm{KKT}(\traj_{j}^\textrm{loc})\}_{j=1}^{\numsafe}, \quad g(\cstate, \theta) > 0 \textrm{ \textbf{OR} } g(\cstate, \theta) \le 0
	\end{array}\hspace{-20pt}
\end{equation}
\end{problem}
\vspace{4pt}

If Problem \ref{prob:query} is infeasible, then $\cstate \in \guarsafe$ or $\cstate \in \guarunsafe$. Solving this problem is akin to querying an oracle about the safety of $\cstate$. The oracle can return that $\cstate$ is guaranteed safe (program infeasible after forcing $\cstate$ to be unsafe), guaranteed unsafe (program infeasible after forcing $\cstate$ to be safe), or unsure (program is feasible despite forcing $\cstate$ to be safe or unsafe).

Since the constraint space is continuous, it is not possible to check via enumeration if each $\cstate \in\guarunsafe$ or $\cstate \in \guarsafe$. To address this, we can check the neighborhood of a constraint state $\cstate_\textrm{query}$ for membership in $\guarunsafe$ by solving the following:

\begin{problem}[Volume extraction]\label{prob:kkt_vol}
\begin{equation}
	\begin{array}{>{\displaystyle}c >{\displaystyle}l >{\displaystyle}l}
			&\\[-18pt]
		\underset{\varepsilon, \cstate_\textrm{near}, \theta, \boldsymbol{\lambda}_{k}^j, \boldsymbol{\lambda}_{\neg k}^j,\boldsymbol{\nu}_{k}^j}{\text{minimize}} & \varepsilon &\\[-2pt]
		\text{subject to} & \{\textrm{KKT}(\traj_{j}^\textrm{loc})\}_{j=1}^{\numsafe} \\
		& \Vert \cstate_\textrm{near} - \cstate_\textrm{query} \Vert_\infty \le \varepsilon \\ 
		& g(\cstate_\textrm{near}, \theta) > 0\\[2pt]
	\end{array}\hspace{-15pt}
\end{equation}
\end{problem}

Intuitively, Problem \ref{prob:kkt_vol} finds the largest box centered at $\cstate_\textrm{query}$ contained within $\guarsafe$. An analogous problem can also be posed to recover the largest hypercube centered at $\cstate_\textrm{query}$ contained within $\guarunsafe$. For some common parameterizations (axis-aligned hyper-rectangles, convex sets), subsets of $\guarsafe$ and $\guarunsafe$ can be even more efficiently recovered by performing line searches or taking convex hulls of guaranteed safe/unsafe states, details of which are in Appendix B of \cite{corl}. Volumes of safe/unsafe space can thus be produced by repeatedly solving Problem \ref{prob:kkt_vol} for different $\cstate_\textrm{query}$.

\vspace{-3pt}
\subsection{KKT relaxation for unions of affine constraints}\label{sec:method_linear}
\vspace{-0pt}

Now, consider when Problem \ref{prob:fwd_prob} involves avoiding an unsafe set $\unsafeset(\theta)$ described by a union of affine constraints:

\vspace{-8pt}\begin{equation}\label{eq:unions_of_affine}
	\unsafeset(\theta) = \bigcup_{i=1}^{N_c} \{\cstate \mid g_{i, \neg k}(\cstate, \theta) > 0\}
\end{equation}\vspace{-8pt}

\noindent where $g_{i, \neg k}(\cstate, \theta)$ is an affine function of $\theta$ for fixed $\cstate$. Unlike in Sec. \ref{sec:method_unions}, formulating the recovery problem like Problem \ref{prob:kkt_union} yields trilinearity in the stationarity condition between continuous variables $\theta$ and $\lambda$ and binary variables $q$, since for the affine case, $\nabla_{\trajxu} g_{i, \neg k}(\cdot, \theta)$ remains a function of $\theta$. As the product of two continuous decision variables cannot be linearized exactly, one must solve a MINLP to recover $\theta$ in this case, which can be inefficient. However, a relaxation which enables querying of guaranteed safeness/unsafeness via Problem \ref{prob:query} can be formulated as a MILP. For legibility, we present the $N_c=1$ case, where there is only one affine constraint (and hence the binary variables $q$ seen in Problem \ref{prob:kkt_union} are all set to 1 and can thus be dropped). Each bilinear term $\lambda_{1,\neg k}^{j,\cstate} \nabla_{\cstate} g_{1, \neg k}(\cstate, \theta)$ is replaced with $l_{1}^{j,\cstate} z_{1,1}^{j,\cstate}$, where $l_{1}^{j,\cstate}$ is a variable which represents the bilinear term and $z_{1,1}^{j,\cstate}$ is an indicator variable encoding that if $z_{1,1}^{j,\cstate}$ is 0, then $\lambda_{1, \neg k}^{j,\cstate}$ must be 0. Hence, by linearizing the bilinear term as such, there is no relaxation gap when the Lagrange multipliers are zero; the only loss is when the Lagrange multipliers are non-zero (i.e. when the demonstration touches the constraint boundaries). In this case, coupling between $\lambda$ and $\theta$ is lost by introducing the $l_{1}^{j,\cstate}$ variables. We further linearize $l_{1}^{j,\cstate} z_{1,1}^{j,\cstate}$ (product of continuous, binary variables) with the same procedure in Sec. \ref{sec:method_unions} by again introducing slack variables $r_{1}^{j,\cstate}$ and constraining them accordingly with \eqref{eq:bilinear_slack}, where the $q_{m,n}^{j,\cstate}$ are replaced with $z_{1,1}^{j,\cstate}$. Putting things together, we can write the following relaxed constraint recovery problem for $N_c=1$:

\begin{problem}[KKT relaxation, affine]\label{prob:kkt_relax}
	\begin{equation}\small
	\hspace{-15pt}\begin{array}{>{\displaystyle}c >{\displaystyle}l >{\displaystyle}l}
				&\\[-10pt]
		\text{find} & \theta, \boldsymbol{\lambda}_{k}^{j,\cstate}, \boldsymbol{\lambda}_{\neg k}^{j,\cstate},\boldsymbol{\nu}_{k}^j,\mathbf{r}_{1}^j, \boldsymbol{\ell}_{1}^j, z_{1,1}^{j,\cstate}, z_{1,2}^{j,\cstate},\ \forall\cstate \in \demj,\ j = 1,...,\numsafe &\\
		\text{s.t.} & \texttt{-------- Primal feasibility --------} \\[-2pt]
		& g_{1,\neg k}(\cstate, \theta) \le 0,\ \forall\cstate \in \demj,\ j=1,\ldots,\numsafe \\[-1pt]
		& \texttt{--- Lagrange mult. nonnegativity ---} \\[-1pt]
		& \boldsymbol{\lambda}_{k}^{j,\cstate} \ge \mathbf{0},\ \forall \cstate \in \demj, \ j = 1,\ldots,\numsafe \\[0pt]
		& \lambda_{1,\neg k}^{j,\cstate} \ge 0,\ \forall \cstate\in \demj, j = 1,\ldots,\numsafe \\[-2pt]
		& \texttt{------ Complementary slackness -----} \\[-1pt]
		& \begin{bmatrix}\lambda_{1,\neg k}^{j,\cstate}\\ \ -g_{1,\neg k}(\cstate, \theta)\end{bmatrix} \le M\begin{bmatrix} z_{1,1}^{j,\cstate}\\z_{1,2}^{j,\cstate}\end{bmatrix},\ z_{1,1}^{j,\cstate}+z_{1,2}^{j,\cstate} \le 1, \\[-1pt]
		& \quad\quad \forall \cstate\in\demj, j = 1,\ldots,N_s\\[-2pt]
		& \texttt{----------- Stationarity -----------} \\[-2pt]
		&\nabla_{\trajxu} c_\task(\traj_{j}^\textrm{loc}) + \boldsymbol{\lambda}_{k}^j \nabla_{\trajxu} \mathbf{g}_{k}(\traj_{j}^\textrm{loc})+\mathbf{r}_{1}^j\\[-0pt]
		&\quad\quad+\mathbf{\boldsymbol{\nu}}_{k}^j \nabla_{\trajxu} \mathbf{h}_{k}(\traj_{j}^\textrm{loc}) = 0,\ j = 1,\ldots,N_s\\
		&\textrm{min}(0, \underline M)\mathbf{1} \le \mathbf{r}_{1}^j \le \overline M\mathbf{1}, \quad \underline M \mathbf{z}_{1,1}^j \le \mathbf{r}_{1}^j \le \overline M \mathbf{z}_{1,1}^j, \\
		&\quad \boldsymbol{\ell}_{1}^j - (\mathbf{1}-\mathbf{z}_{1,1}^j)\overline M \le \mathbf{r}_{1}^j \le \boldsymbol{\ell}_{1}^j - (\mathbf{1}-\mathbf{z}_{1,1}^j)\underline M, \\
		&\quad \mathbf{r}_{1}^j \le \boldsymbol{\ell}_{1}^j + (\mathbf{1}-\mathbf{z}_{1,1}^j)\overline M,\quad j = 1,\ldots,N_s\\[3pt]
	\end{array}\hspace{-40pt}
\end{equation}
\end{problem}
\noindent where $\mathbf{r}_1^j$, $\mathbf{z}_{1,1}^j$, $\boldsymbol{\ell}_1^j$ denote horizontal concatenation of $r_1^{j,\cstate}$, $z_{1,1}^{j,\cstate}$, $l_1^{j,\cstate}$ over $\cstate$. The case where the constraint is a union of affine constraints yields quadrilinearity and can be handled similarly, requiring one extra step to linearize the products of binary variables $q_{m}^{j,\cstate}$ and $z_{1,1}^{j,\cstate}$, which can be done exactly.

While Problem \ref{prob:kkt_relax} cannot recover the constraint parameter $\theta$ directly, one can still check if a constraint state is guaranteed safe/unsafe using Problem \ref{prob:query} (see Theorem \ref{thm:kkt_relax} for reasoning).
\vspace{-15pt}
\subsection{Unknown constraint parameterization}\label{sec:method_unknown}
\vspace{-2pt}

In many applications, we may not know a constraint parameterization a priori. However, complex unsafe/safe sets can often be approximated as the union of many simple unsafe/safe sets. Thus, we adapt the method in \cite{corl} for incrementally growing a parameterization based on the complexity of the provided demonstrations. More precisely, suppose the true parameterization $g(\cstate, \theta)$ of the unsafe set $\unsafeset(\theta) = \{\cstate \mid g(\cstate, \theta) > 0\}$ is unknown but can be exactly/approximately written as the union of $\numbox$ simple sets $\unsafeset(\theta) \approxeq \bigcup_{i=1}^{\numbox}\{\cstate\mid g_s(\cstate,\theta_i) > 0\}\doteq \bigcup_{i=1}^{\numbox} \unsafeset(\theta_i)$. Each simple set $\unsafeset(\theta_i)$ has a known parameterization $g_s(\cdot, \cdot)$ but $\numbox$, the minimum number of simple sets needed to reconstruct $\unsafeset$, is unknown. We can estimate a lower bound on $\numbox$, $\underline N$, by incrementally adding simple sets until Problem \ref{prob:kkt_opt} is feasible (i.e. there exists a sufficiently complex constraint which can satisfy the demonstrations' KKT conditions). Issues with conservativeness of the recovered constraint when $\underline N < \numbox$ are discussed in \cite{corl} and also hold here, which we omit for brevity.

\vspace{0pt}
\subsection{Handling cost function uncertainty}\label{sec:method_cost}
\vspace{0pt}

We now extend the KKT conditions presented in \eqref{eq:kkt} and Problems \ref{prob:kkt_opt} and \ref{prob:kkt_subopt} to learn constraints with parametric uncertainty in the cost function (i.e. if $\thetac$ in Problem \ref{prob:fwd_prob} is unknown). To address this, the first term in the stationarity condition \eqref{eq:kkt_stat} must be changed to $\nabla_{\trajxu} c_\task(\traj_{j}^\textrm{loc}, \textcolor{red}{\thetac})$. Then, if $c_\task(\cdot, \gamma)$ is affine in $\gamma$, $\thetac$ can be found using a MILP.

Querying/volume extraction holds just as before; the only difference is that $\thetac$ is now a decision variable in Problem \ref{prob:query}/\ref{prob:kkt_vol}. Note we are extracting constraint states that are guaranteed safe/unsafe for all possible cost parameters; that is, we are extracting safe/unsafe sets that are robust to cost uncertainty.

We summarize what we can solve for when using various parameterizations. For the exact cases, we can solve for $\theta$/$\thetac$, but when relaxing, we can only solve for $\safeset$/$\unsafeset$ via queries. Note the constraint/cost can be nonlinear in $\cstate$ without inducing relaxation, though it precludes usage of Problem \ref{prob:kkt_vol} (as $\cstate$ is a decision variable in the latter, but not the former):

\begin{centering}\vspace{4pt}
\footnotesize\begin{tabular}{ c |c| c|c}
 Constraint param. & Cost param. &\hspace{-8pt}Recover $\theta, \thetac$?\hspace{-8pt} & $\guarsafe/\guarunsafe$\hspace{-5pt} \\\hline\hline
 \hspace{-5pt}$\theta$: form of \eqref{eq:unions_of_offsets}; $\cstate$: affine & \hspace{-5pt}$\thetac$: affine; $\cstate$: nonlin.\hspace{-5pt} & Yes: Prob. \ref{prob:kkt_union} & \hspace{-5pt}Prob. \ref{prob:query}/\ref{prob:kkt_vol}\hspace{-5pt} \\\hline
 \hspace{-5pt}$\theta$: form of \eqref{eq:unions_of_offsets}; $\cstate$: nonlin.\hspace{-5pt} & \hspace{-5pt}$\thetac$: affine; $\cstate$: nonlin.\hspace{-5pt} & Yes: Prob. \ref{prob:kkt_union} & \hspace{-5pt}Prob. \ref{prob:query}\hspace{-5pt} \\\hline
 \hspace{-5pt}$\theta$: form of \eqref{eq:unions_of_affine}; $\cstate$: nonlin.\hspace{-5pt} & \hspace{-5pt}$\thetac$: affine; $\cstate$: nonlin.\hspace{-5pt} & No & \hspace{-5pt}Prob. \ref{prob:query}\hspace{-5pt} \\\hline
\end{tabular}\vspace{4pt}
\end{centering}
\noindent This only describes what we can solve for; the actual accuracy of the recovered $\theta$/$\gamma$ and the size of the recovered $\guarsafe$/$\guarunsafe$ depends on how informative the demonstrations are, i.e. the demonstrations should interact with the constraint.

\vspace{-6pt}
\subsection{Applications to safe planning}\label{sec:method_planning}
\vspace{-2pt}

As learned constraints can be reused for novel tasks with the same safety requirements, we end this section by describing how our method can be used within a planner to guarantee the safety of trajectories planned for such tasks. Recall that Problems \ref{prob:query} and \ref{prob:kkt_vol} can be used to query if a constraint state $\cstate$ or a region around $\cstate$ is guaranteed safe/unsafe. The planner can use this information by either:
\begin{itemize}
	\item Extracting an \textit{explicit} representation of the constraint by repeatedly solving Problem \ref{prob:kkt_vol} for different $\cstate$ to cover $\safeset$ and $\unsafeset$. Denote these extracted sets as $\hat\safeset \subseteq \safeset$ and $\hat\unsafeset \subseteq \unsafeset$ (the conservativeness of our method is proved in Sec. \ref{sec:conservativeness}). Then, $\hat\safeset$ can be passed to a planner and quickly used for constraint/collision checking via set-containment checks. A planned trajectory is \textit{guaranteed safe} if each state on it lies in $\hat\safeset$, since $\hat\safeset$ is contained in true safe set $\safeset$. If $\hat\safeset$ is small and the planner cannot find a feasible trajectory, we can at least guarantee that a trajectory is \textit{not definitely unsafe} it it does not intersect with $\hat\unsafeset$, as $\hat\unsafeset$ is contained in the true unsafe set $\unsafeset$.
	\item Extracting an \textit{implicit} representation of the constraint by solving Problem \ref{prob:query} as needed by the planner. This may be less computationally efficient than the explicit case, but we demonstrate in Sec. \ref{sec:res_quad} that we still achieve reasonable planning times for a 7-DOF arm.
\end{itemize}

\vspace{-8pt}
\section{Theoretical Analysis}
\vspace{-2pt}
In this section, we prove that our method provides a conservative estimate of the guaranteed learned safe/unsafe sets $\guarsafe, \guarunsafe$ (Sec. \ref{sec:conservativeness}) and prove learnability results using locally-optimal demonstrations (Sec. \ref{sec:learnability}).

\subsection{Conservativeness}\label{sec:conservativeness}
\vspace{-2pt}

\begin{definition}[Implied unsafe/safe set]
For some set $\mathcal{B} \subseteq \Theta$, let $I_{\neg s}(\mathcal{B}) \doteq \bigcap_{ \theta \in \mathcal{B} } \{x\ |\ g(x, \theta) > 0 \}$ be the set of states implied unsafe by restricting the parameter set to $\mathcal{B}$, i.e. $I_{\neg s}(\mathcal{B})$ is the set of states that all $\theta \in \mathcal{B}$ mark as unsafe. Similarly, let $I_{s}(\mathcal{B}) \doteq \bigcap_{ \theta \in \mathcal{B} } \{x\ |\ g(x, \theta) \le 0 \}$ be the set of states implied safe by restricting the parameter set to $\mathcal{B}$.
\end{definition}

We further introduce the following lemma:

\begin{lemma}[Lemma C.1 in \cite{corl}]\label{lem:contain}
	Suppose $\mathcal{B} \subseteq \mathcal{\hat B}$, for some other set $\mathcal{\hat B}$. Then, $I_{\neg s}(\mathcal{\hat B}) \subseteq I_{\neg s}(\mathcal{B})$ and $I_{s}(\mathcal{\hat B}) \subseteq I_{s}(\mathcal{B})$.
\end{lemma}

\begin{theorem}[Conservativeness of Problem \ref{prob:kkt_opt}]\label{thm:conserv_kkt_opt}
	Suppose the constraint parameterization $g(\state, \theta)$ is known exactly. Then, extracting $\guarsafe$ and $\guarunsafe$ (as defined in \eqref{eq:guarsafe} and \eqref{eq:guarunsafe}, respectively) from the feasible set of Problem \ref{prob:kkt_opt} projected onto $\Theta$ (denoted as $\feas$) returns $\guarunsafe \subseteq \unsafeset$ and $\guarsafe \subseteq \safeset$.
\end{theorem}
\begin{proof}
	We first prove that $\guarunsafe\subseteq \unsafeset$. Suppose that there exists $\state \in \guarunsafe$ such that $\state \notin \unsafeset$. Then by definition, for all $\theta \in \feas$, $g(x, \theta) > 0$.  However, we know that all locally-optimal demonstrations satisfy the KKT conditions with respect to the true parameter $\theta^*$; hence, $\theta^* \in \feas$. Then, $\state \in \unsafeset(\theta^*)$. Contradiction.	Similar logic holds for proving that $\guarsafe \subseteq \safeset$. Suppose that there exists $\state \in \guarsafe$ such that $\state \notin \safeset$. Then by definition, for all $\theta \in \feas$, $g(x, \theta) \le 0$.  However, we know that all locally-optimal demonstrations satisfy the KKT conditions with respect to the true parameter $\theta^*$; hence, $\theta^* \in \feas$. Then, $\state \in \safeset(\theta^*)$. Contradiction.
\end{proof}
\begin{rem}
Unfortunately, it is difficult to guarantee conservativeness when using suboptimal demonstrations (solving Problem 	\ref{prob:kkt_subopt}), as the relationship between cost suboptimality and KKT violation is generally unknown. However, we note in practice that the $\guarsafe, \guarunsafe$ recovered using suboptimal demonstrations still tend to be conservative (see Sec. \ref{sec:res_arm}).
\end{rem}

\begin{theorem}[Conservativeness of Problem \ref{prob:kkt_relax}]\label{thm:kkt_relax}
	Suppose the constraint parameterization $g(\state, \theta)$ is known exactly. Then, extracting $\guarsafe$ and $\guarunsafe$ (as defined in \eqref{eq:guarsafe} and \eqref{eq:guarunsafe}, respectively) from the feasible set of Problem \ref{prob:kkt_relax} (denoted as $\feas$) returns $\guarunsafe \subseteq \unsafeset$ and $\guarsafe \subseteq \safeset$.
\end{theorem}
\begin{proof}
	Denote the $\Theta$-projected feasible set of the original unrelaxed problem (i.e. variables $r_{i}$ are not introduced and the bilinear terms between $\theta$ and $\lambda$ remain) as $\feas_\textrm{MINLP}$ and the $\Theta$-projected feasible set of Problem \ref{prob:kkt_relax} as $\feas$. Using the logic in Theorem \ref{thm:conserv_kkt_opt}, extracting $\guarsafe$ and $\guarunsafe$ from $\feas_\textrm{MINLP}$ yields $\guarsafe \subseteq \safeset$ and $\guarunsafe \subseteq \unsafeset$ (since $\feas_\textrm{MINLP}$ is the true feasible set, like assumed in Theorem \ref{thm:conserv_kkt_opt}). Furthermore, $\feas_\textrm{MINLP} \subseteq \feas$, since relaxing the bilinear terms to linear terms in Problem \ref{prob:kkt_relax} expands the feasible set compared to the unrelaxed problem. By definition, $\guarsafe = I_s(\feas_\textrm{MINLP})$ and $\guarunsafe = I_{\neg s}(\feas_\textrm{MINLP})$, and via Lemma \ref{lem:contain}, $I_s(\feas) \subseteq I_s(\feas_\textrm{MINLP})$ and $I_{\neg s}(\feas) \subseteq I_{\neg s}(\feas_\textrm{MINLP})$. Hence, $I_s(\feas) \subseteq \safeset$ and $I_{\neg s}(\feas) \subseteq \unsafeset$.
\end{proof}

\begin{rem}
	For brevity, we omit conditions on $M, \underline M, \overline M$ for conservativeness; it is well-known that this is achieved by choosing the big-M constants to be sufficiently large \cite{bertsimas}.
\end{rem}

\vspace{-4pt}
\subsection{Global vs local learnability}\label{sec:learnability}
\vspace{-0pt}

\begin{definition}[Local learnability]
	A state $\state \in \unsafeset$ is locally learnable if there exists \textbf{any} set of $\numsafe$ locally-optimal demonstrations, where $\numsafe$ may be infinite, such that $\state \in \mathcal{I}_{\neg s}(\feas)$, where $\feas$ is the $\Theta$-projected feasible set of Problem \ref{prob:kkt_opt}. We also define the \textit{locally learnable} set of unsafe states $\guarunsafe^{\textrm{loc},*}$ as the union of all locally learnable states.
\end{definition}

\begin{definition}[Global learnability]
		A state $\state \in \unsafeset$ is globally learnable if there exists \textbf{any} set of $\numsafe$ globally-optimal demonstrations and $\numunsafe$ sampled strictly lower-cost (and hence unsafe) trajectories, where $\numsafe$ and $\numunsafe$ may be infinite, such that $\state \in \mathcal{I}(\feas_\textrm{glo})$, where $\mathcal{F}_\textrm{glo}$ is the feasible set of Problem 2 in \cite{corl} (which recovers a constraint consistent with the demonstrations and sampled unsafe trajectories). Accordingly, we define the \textit{globally learnable} set of unsafe states $\guarunsafe^{\textrm{glo},*}$ as the union of all globally learnable states.
\end{definition}

Note that a safe state $\state_s \in \safeset$ can always be learned guaranteed safe, as there always exists a safe globally-optimal or locally-optimal demonstration passing through $\state_s$. Armed with these definitions, we show the following:

\begin{theorem}[Global vs local]\label{thm:global_vs_local}
	Suppose the initial constraint parameter set $\Theta$ is identical for both the local and global problems. Then, $\guarunsafe^{\textrm{loc},*} \subseteq \guarunsafe^{\textrm{glo},*}$.
\end{theorem}
\begin{proof}
	Any globally-optimal demonstration must also satisfy the KKT conditions, as it is also locally-optimal. Further conditions (in the form of lower-cost trajectories being infeasible) must be imposed on a constraint parameter for it to be globally-optimal. Hence, $\feas_\textrm{glo} \subseteq \feas$. By Lemma \ref{lem:contain}, $\mathcal{I}(\feas) \subseteq \mathcal{I}(\feas_\textrm{glo})$, and thus $\guarunsafe^{\textrm{loc},*} \subseteq \guarunsafe^{\textrm{glo},*}$.
\end{proof}

Note that Theorem \ref{thm:global_vs_local} holds in the limit of having sampled \textbf{all} unsafe trajectories. In practice, the sampling is nowhere near complete, especially for nonlinear dynamics. We see in these cases (Sec. \ref{sec:res_quad}) that our KKT-based method learns more compared to sampling-based techniques. Finally, we note that cost function uncertainty can only decrease learnability, as it enlarges the feasible set of Problem \ref{prob:kkt_opt}.
\vspace{-2pt}
\section{Results}
\vspace{-1pt}
We show our method, first on 2D examples (Sec. \ref{sec:res_toy}) for intuition, and then on high-dimensional 7-DOF arm (Sec. \ref{sec:res_arm}) and quadrotor (Sec. \ref{sec:res_quad}) constraint-learning problems (see the accompanying video for experiment visualizations). All computation times are recorded on a laptop with a 3.1 GHz Intel Core i7 processor and 16 GB RAM.

\vspace{-4pt}
\subsection{2D examples}\label{sec:res_toy}
\vspace{-1pt}
\noindent\textit{\underline{Global vs. local}:} Assuming global demonstration optimality can enlarge $\guarsafe/\guarunsafe$ compared to assuming local optimality (Theorem \ref{thm:global_vs_local}). In this example, we show some common differences in the learned constraints when assuming global/local optimality. Consider a 2D kinematic system $\chi_{t+1} = \chi_t + u_t$, $\chi = [x, y]^\top$, $\Vert u_t \Vert \le 1$ avoiding the pink obstacle in Fig. \ref{fig:global_vs_local}. We use an axis-aligned box constraint parameterization. In Fig. \ref{fig:global_vs_local} (left), by assuming the demonstrations (cyan, green) are globally-optimal and sampling lower-cost trajectories (the middle state on each trajectory is plotted in red), the hatched area is implied guaranteed unsafe, as any axis-aligned box containing the sampled unsafe states (in red) must also contain the hatched area. In contrast, assuming local optimality gives us zero volume learned guaranteed safe/unsafe, as a measure-zero horizontal line obstacle (orange dashed) can make the demonstrations locally-optimal: as the line supports the middle state on each demonstration, the cost cannot be locally improved. In Fig. \ref{fig:global_vs_local}, center, we show a case where there is no gap in learnability: without assuming a parameterization, the demonstrations can be explained by two horizontal line obstacles, but together with the box parameterization, we recover $\guarsafe = \safeset$ and $\guarunsafe = \unsafeset$. Fig. \ref{fig:global_vs_local} (right) shows that assuming global optimality may result in non-conservative constraint recovery (e.g. if the dotted red line were a sampled unsafe trajectory), while a horizontal line obstacle (orange dashed line) can explain local optimality of the demonstration, yielding conservative constraint recovery.

\begin{figure}
\centering
\vspace{5pt}
\includegraphics[width=\linewidth]{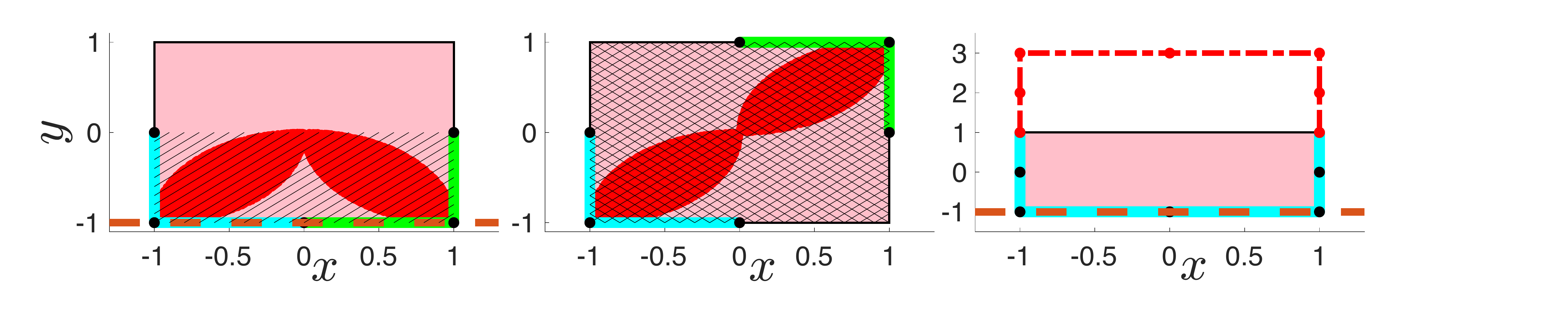}
\vspace{-20pt}
\caption{Left: local learns less than global. Center: local learns the same as global. Right: global recovers non-conservative solution. Red: sampled unsafe trajectories. Pink: true constraint. Green/cyan: demonstrations. \vspace{-8pt}}
\label{fig:global_vs_local}
\end{figure}

\noindent\textit{\underline{Effects of cost uncertainty}:}
We show that learnability under cost uncertainty is more related to the possible behaviors that a cost uncertainty set can represent, rather than the actual size of the cost parameter space. For the demonstrations/constraint in the center plot of Fig. \ref{fig:global_vs_local}, consider the following cost uncertainty sets: A) $c(\traj) = \sum_{t=1}^{T-1} \gamma_1(x_{t+1} - x_t)^2 + \gamma_2(y_{t+1} - y_t)^2$, where $\gamma_i \in [-5, 5], i = 1, 2$ and B) $c(\traj) = \sum_{k=1}^{10}\sum_{t=1}^{T-1} \big(\gamma_{1,k}(x_{t+1} - x_t)^{2k} + \gamma_{2,k}(y_{t+1} - y_t)^{2k}\big)$, where $\gamma_{i,k}\in[0.001, 5]$ for all $i, k$. While Set A has a much smaller parameter space compared to Set B (2 vs 20 parameters), allowing $\gamma_1, \gamma_2$ to take negative values enables the case where the demonstrator wants to maximize path length (i.e. set $\gamma_1 = \gamma_2 = -1$). For fixed start/goal states and control constraints, the observed demonstrations are actually locally-optimal with respect to a cost function which maximizes path length in an environment with no box state space constraint. Hence, for Set A, our method returns $\guarsafe = \guarunsafe = \emptyset$. In contrast, while Set B has a much larger parameter space, the range of allowable behaviors is small (all cost terms must penalize path length). Thus, despite the large cost parameter space, $\guarsafe = \safeset$ and $\guarunsafe = \unsafeset$.

\begin{figure}
\centering
\hspace{-10pt}\floatbox[{\capbeside\thisfloatsetup{capbesideposition={right,center},capbesidewidth=3.1cm}}]{figure}[\FBwidth]
{\hspace{-10pt}\caption{Nonlinear constraint. Blue: true constraint boundary. Red/green states: learned in $\guarunsafe/\guarsafe$. Purple/orange: two demonstrations.}\label{fig:nonlinear}}
{\includegraphics[width=5.5cm]{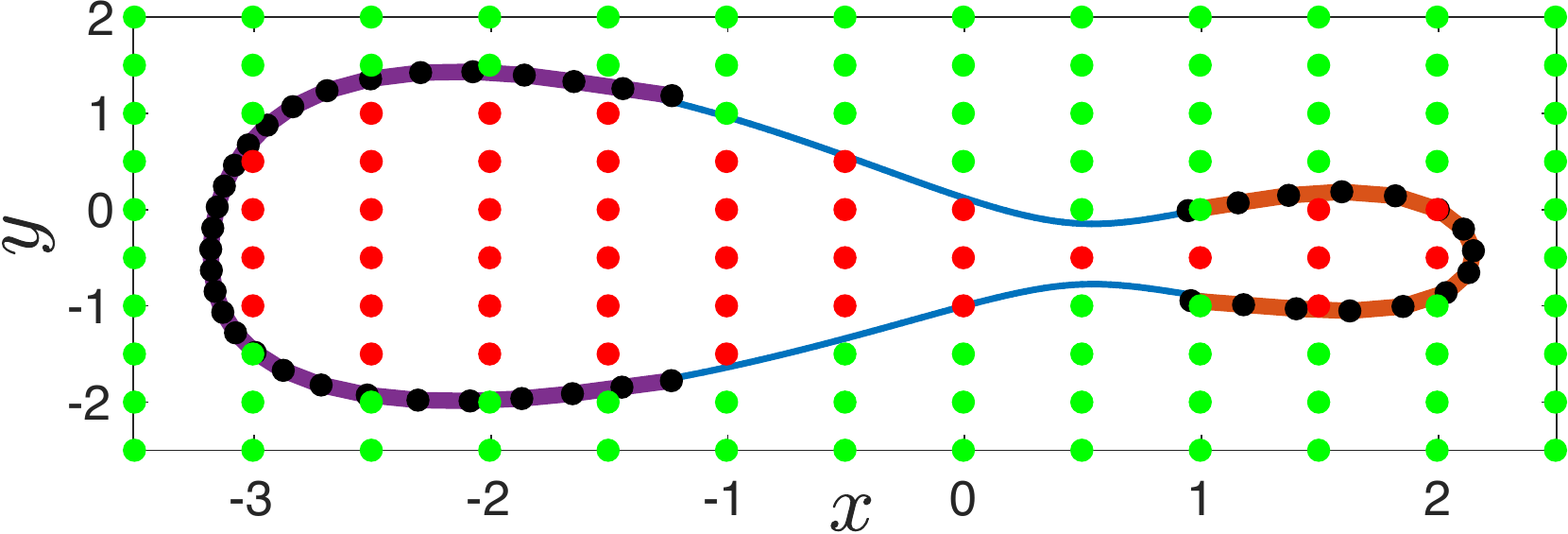}}
\vspace{-20pt}
\end{figure}

\noindent\textit{\underline{Nonlinear constraint}:}
We emphasize that while our method requires an affine parameterization in the \textit{constraint parameters}, constraints that are nonlinear in the \textit{state} can still be learned. Consider a parameterization $g_{1,\neg k}(\cstate, \theta) = 2(x^4 + y^4) - 5(x^3 + y^3) + 5(x - 1)^3 + 5(y + 1)^3- \theta$, which yields a highly nonlinear state space constraint. With two demonstrations for $\theta = 2$ (see Fig. \ref{fig:nonlinear}), $\guarsafe=\safeset$ and $\guarunsafe=\unsafeset$.

\vspace{-6pt}
\subsection{7-DOF arm}\label{sec:res_arm}
\vspace{6pt}

\vspace{-5pt}\begin{figure}
\centering
\vspace{-13pt}
\hspace{-5pt}\floatbox[{\capbeside\thisfloatsetup{capbesideposition={right,center},capbesidewidth=1.5cm}}]{figure}[\FBwidth]
{\hspace{-12pt}\caption{Left: demonstrations for bartender example. Right: novel trajectories planned with learned constraint.}\label{fig:arm_viz}}
{\includegraphics[width=7cm]{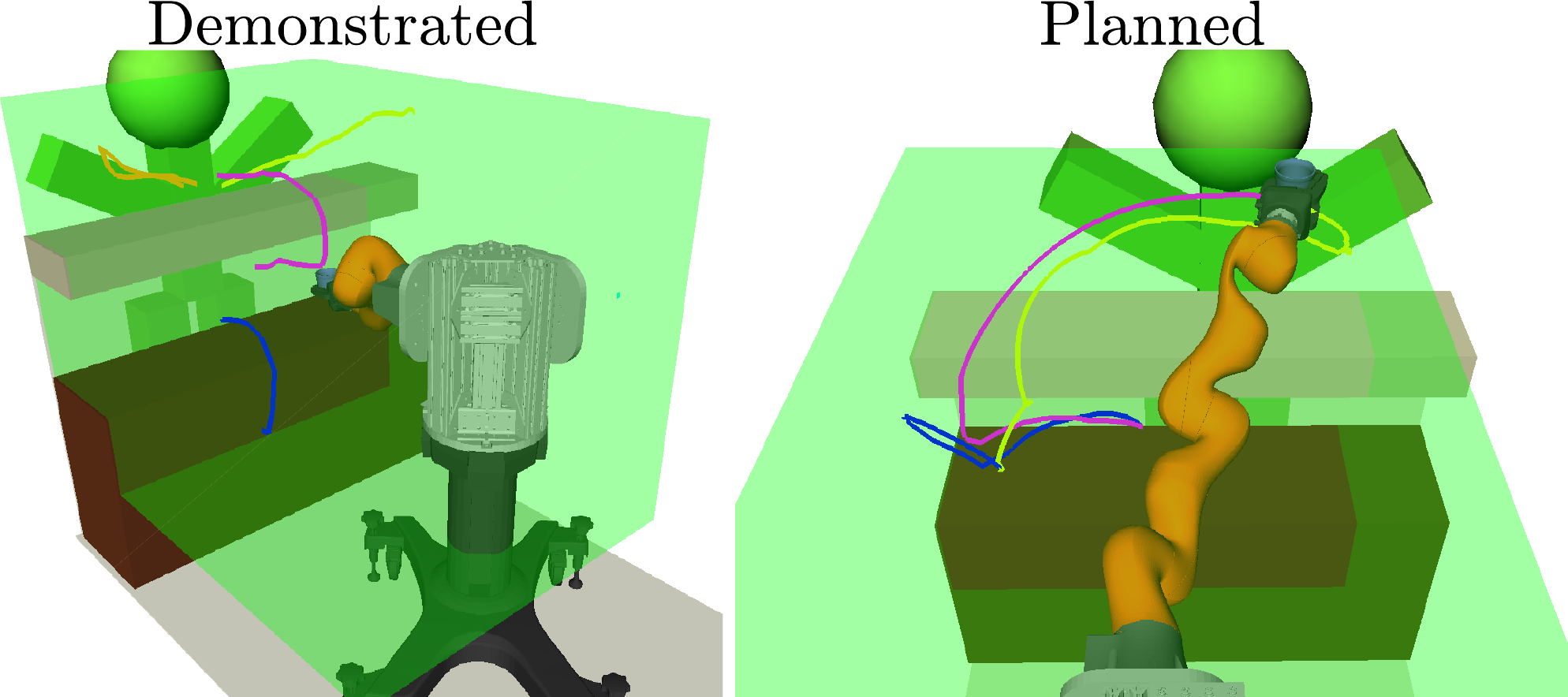}}
\end{figure}
\vspace{-3pt}

\begin{figure}
\centering
\vspace{5pt}
\includegraphics[width=\linewidth]{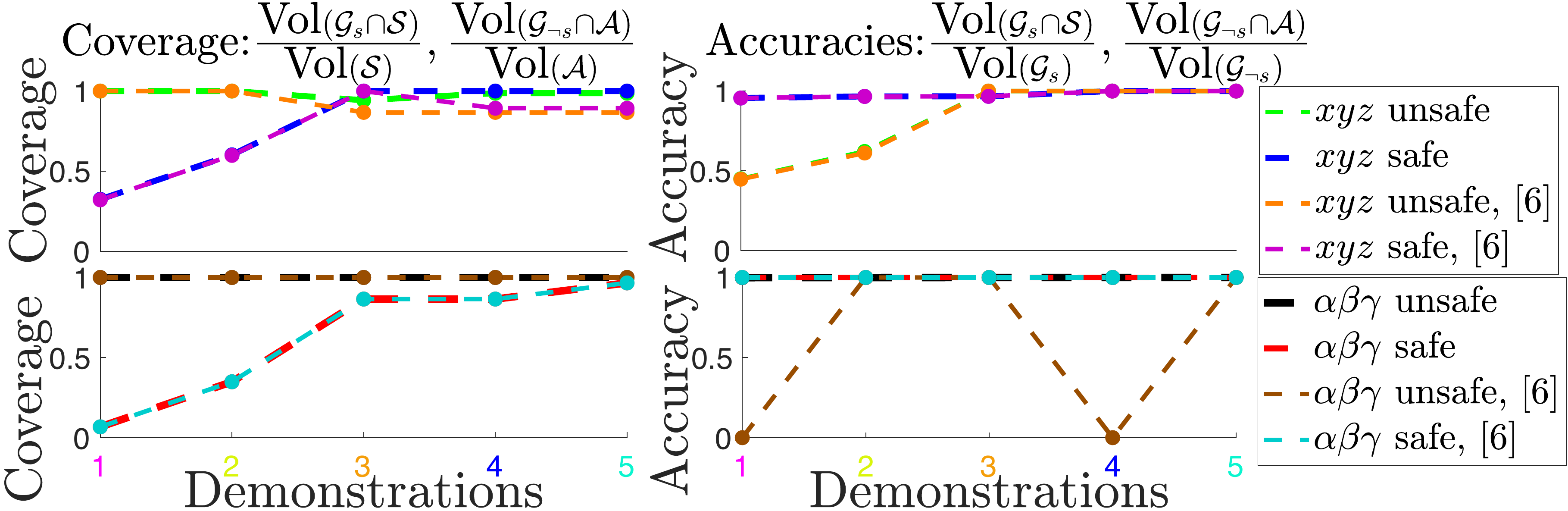}
\vspace{-16pt}
\caption{Arm bartender statistics; x-axis color-coded with demos in Fig. \ref{fig:arm_viz}.}
\label{fig:arm_stats}
\vspace{-5pt}
\end{figure}

\begin{figure}\vspace{-6pt}
\centering
\hspace{-5pt}\floatbox[{\capbeside\thisfloatsetup{capbesideposition={right,center},capbesidewidth=1.5cm}}]{figure}[\FBwidth]
{\hspace{-12pt}\caption{Left: arm demos for ellipse example. Right: novel trajectories planned using learned constraint.}\label{fig:arm_ellipse}}
{\includegraphics[width=6.6cm]{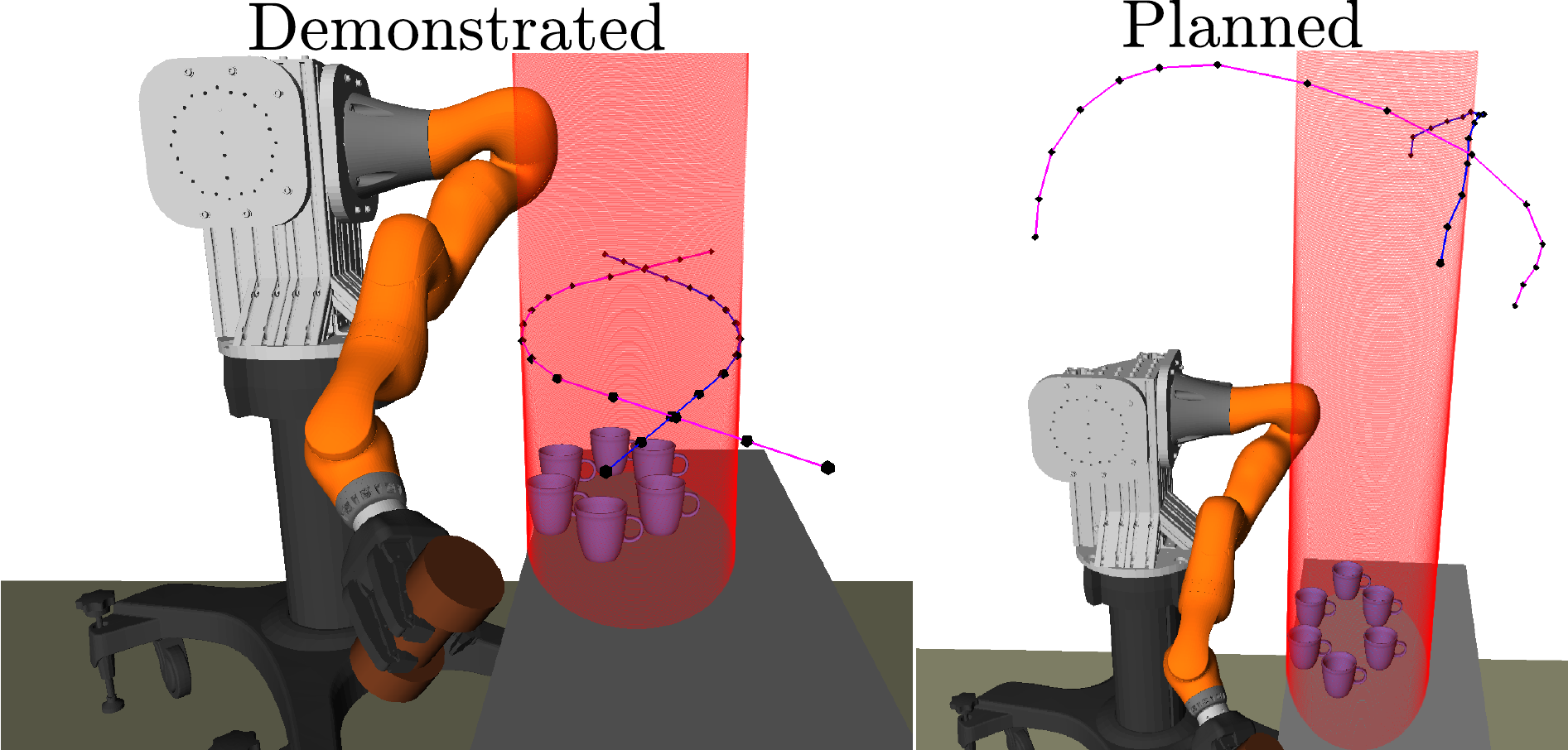}}
\vspace{-22pt}
\end{figure}

\begin{figure}
\centering
\vspace{-18pt}
\hspace{-5pt}\floatbox[{\capbeside\thisfloatsetup{capbesideposition={right,center},capbesidewidth=1.5cm}}]{figure}[\FBwidth]
{\hspace{-12pt}\caption{Left: quadrotor demonstrations. Right: novel planned trajectories.}\label{fig:quad_demos}}
{\includegraphics[width=6.7cm]{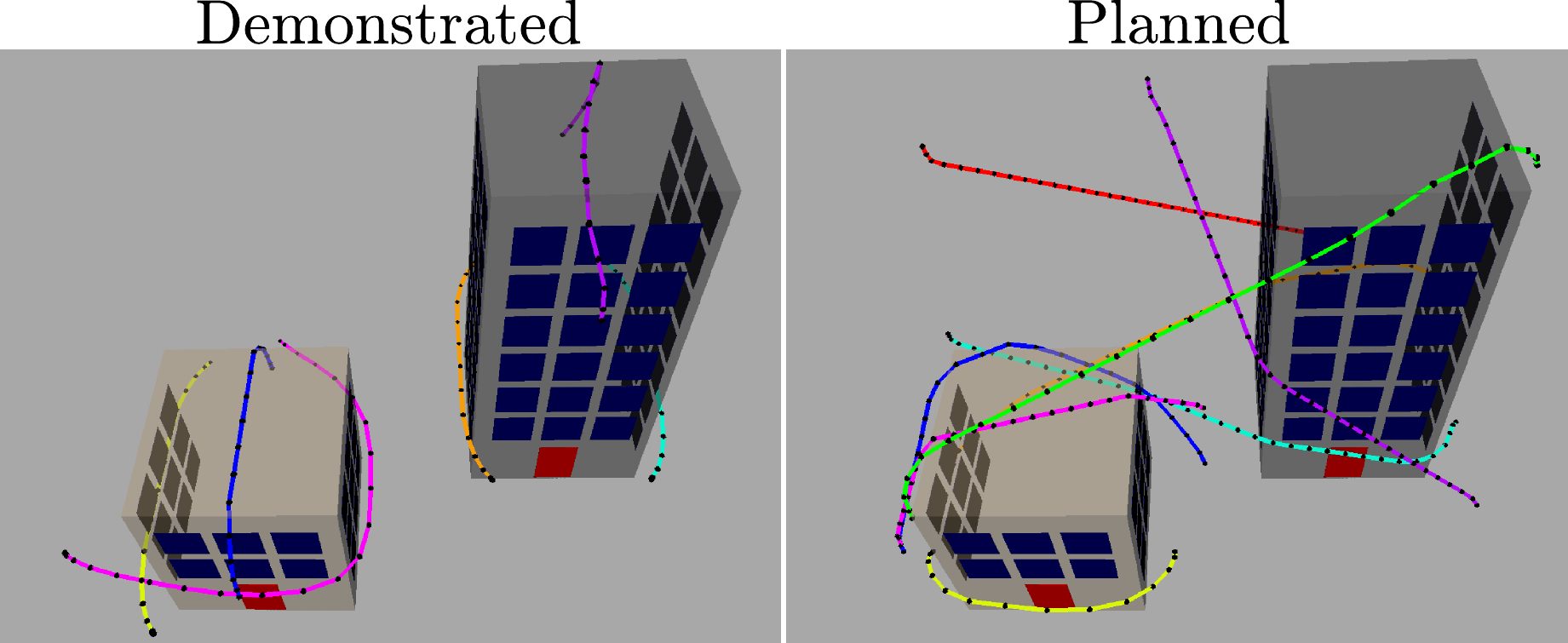}\vspace{0pt}}
\end{figure}

\begin{figure*}
\vspace{4pt}
\centering
\hspace{-0pt}\floatbox[{\capbeside\thisfloatsetup{capbesideposition={right,center},capbesidewidth=3.2cm}}]{figure}[\FBwidth]
{\hspace{-10pt}\caption{Quadrotor statistics: coverage and accuracy for $\guarsafe$, $\guarunsafe$. Demonstration axis is color coded with the demonstrations shown on the left in Fig. \ref{fig:quad_demos}.}\label{fig:quad_stats}}
{\includegraphics[width=14cm]{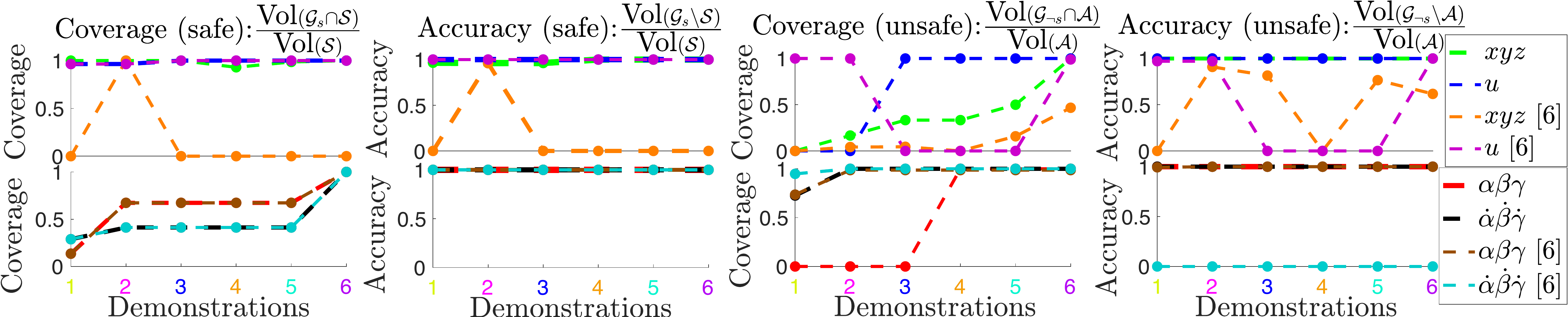}\vspace{-12pt}}
\end{figure*}

\noindent\textit{\underline{Robot bartender}:} Consider a 7-DOF Kuka iiwa robot bartender which must deliver a drink from the bar cabinet (Fig. \ref{fig:arm_viz}, brown box) or from another bartender to a customer at the counter (Fig. \ref{fig:arm_viz}, gray box). To do this, the arm must satisfy an end-effector pose constraint to avoid spilling the drink, and the swept volume of the arm must not collide with the bar furniture while satisfying proxemics constraints (Fig. \ref{fig:arm_viz}, green box) with respect to the customer. We use a kinematic model of the arm, $j_{t+1}^i = j_t^i + u_t^i, i = 1,...,7$, where $\Vert u_t \Vert_2^2 \le 0.8$ for all $t$. Five suboptimal human demonstrations, optimizing joint-space path length $c(\traj) = \sum_{i=1}^{T-1} \Vert j_{t+1} - j_t \Vert_2^2$, are captured in a virtual reality environment using an HTC Vive. The proxemics constraints encoded in the demonstrations (Fig. \ref{fig:arm_viz}, left) disallow the arm from getting too close to the customer and from making large sweeping motions from the left, right, and particularly the top, as the customer can perceive such motions as aggressive. We aim to learn these 15 constraint parameters: $[\underline x^\textrm{ctr}, \underline z^\textrm{ctr}, \overline z^\textrm{ctr}]$ (unknown extents of the bar top), $[\underline x^\textrm{cab}, \overline z^\textrm{cab}]$ (unknown extents of the bar cabinet), $[\underline \alpha, \overline \alpha, \underline \beta, \overline \beta, \underline \gamma, \overline \gamma]$ (unknown pose constraint), and $[\overline x^\textrm{prox}, \underline y^\textrm{prox}, \overline y^\textrm{prox}, \overline z^\textrm{prox}]$ (box proxemics constraint). 

The constraint parameters are recovered using the suboptimal analogue of Problem \ref{prob:kkt_union} (i.e. using the objective function of Problem \ref{prob:kkt_subopt}), taking 17.2 seconds to solve when using all demonstrations. For tractability, we approximate the swept volume constraint by sampling 18 points on the volume of the arm, mapping them through the arm's forward kinematics, and ensuring that the resulting points are consistent with the obstacle avoidance/proxemics constraint parameters. We solve Problems \ref{prob:query}/\ref{prob:kkt_vol} to extract the learned guaranteed safe/unsafe sets $\guarsafe$/$\guarunsafe$, where each query takes 16.4/12.1 seconds on average over 10 queries. Fig. \ref{fig:arm_stats} shows the coverage of $\guarsafe$/$\guarunsafe$ compared to the true safe/unsafe sets $\safeset$/$\unsafeset$, as well as the accuracy of the claimed safe/unsafe sets (Fig. \ref{fig:arm_stats}). We use the sampling-based approach described in \cite{corl} with Problem 3 of \cite{corl} as a baseline. We note that \cite{corl} will have difficulty with this example, since the swept volume constraint scales the number of decision variables in Problem 3 of \cite{corl} by a factor of 18; this limits the number of trajectory samples which can be tractably used in the constraint-recovery problem. Despite this, the pose constraint is learned fairly well by both the baseline and our method, though the baseline experiences accuracy dips due to suboptimality causing some safe lower-cost trajectories. However, the baseline performs poorly on the position constraints, as it does not learn that the bar top or the bar cabinet are unsafe and does not fully learn the safe set, due to insufficient trajectory samples. In contrast, our KKT-based approach recovers $\guarsafe = \safeset$ and $\guarunsafe = \unsafeset$. Finally, we extract volumes of guaranteed safe space using the procedure in Sec. \ref{sec:method_volume} and provide the extracted constraint to the CBiRRT planner \cite{tsr}, generating the novel safe trajectories in Fig. \ref{fig:arm_viz} (right). This experiment suggests that when Problem \ref{prob:fwd_prob} has many constraints (in this case, due to the swept volume), sampling trajectories leads to worse scalability and worse constraint-recovery performance compared to our KKT-based approach.

\noindent\textit{\underline{Elliptical end-effector constraint}:} This example is meant to demonstrate the efficacy of using Problem \ref{prob:query} in the planning loop. Suppose the arm is manipulating a heavy object near some glassware. For safety, the end effector's center of mass is constrained to lie outside an elliptical cylinder containing the glassware: $\chi_t^\top A(\theta) \chi_t - 2b(\theta)^\top \chi_t + c(\theta) > 0$, where $\chi_t = [x_t, y_t]^\top, A(\theta) = \textrm{diag}([0.5, 2]), b(\theta) = [0, 1.1]^\top$, and $c(\theta) = 0.505$. We modify Problem \ref{prob:query} to use the affine-relaxed KKT conditions, and solving this problem using two demonstrations (Fig. \ref{fig:arm_ellipse}, left) is enough to recover $\guarsafe = \safeset$, $\guarunsafe = \unsafeset$ via queries. To plan novel constraint-satisfying trajectories, we use STOMP \cite{stomp}, where the usual collision/constraint checker is replaced with Problem \ref{prob:query}. We show two planned trajectories (Fig. \ref{fig:arm_ellipse}, right), where the planning times were 2 and 6 minutes. Averaged over 10 different queries, solving Problem \ref{prob:query} takes 0.073 seconds. We note that this can be sped up by warm-starting Problem \ref{prob:query} with the results of previous queries (since like many trajectory optimizers, STOMP samples points near previous iterates).

\vspace{-6pt}
\subsection{Quadrotor}\label{sec:res_quad}
\vspace{-3pt}

We consider the scenario of a quadrotor carrying a delicate payload in an urban environment (see Fig. \ref{fig:quad_demos}). Accordingly, the quadrotor is constrained to not collide with surrounding buildings (i.e. $(x, y, z) \notin ([\underline x_1, \overline x_1]\times [\underline y_1, \overline y_1]\times [0, \overline z_1]) \vee ([\underline x_2, \overline x_2]\times[\underline y_2, \overline y_2]\times [0, \overline z_2])$), satisfy control constraints $\Vert u_t \Vert \le \overline U$, pose constraints $\alpha \in [\underline \alpha, \overline \alpha], \beta \in [\underline \beta, \overline \beta], \gamma \in [\underline \gamma, \overline \gamma]$, and angular velocity constraints $\dot \alpha\in [\underline{\dot\alpha}, \overline{\dot\alpha}], \dot \beta\in [\underline{\dot\beta}, \overline{\dot\beta}], \dot\gamma \in [\underline{\dot\gamma}, \overline{\dot\gamma}]$. In this problem, we aim to recover all of these constraints (23 unknown constraint parameters total) using the six demonstrations in Fig. \ref{fig:quad_demos} (left) by solving Problem \ref{prob:kkt_union}, which takes 19.4 seconds when using all demonstrations. We start from a single box parameterization for each constraint and detect from infeasibility that another box should be added to the position constraint parameterization (see Sec. \ref{sec:method_unknown}). The demonstrations are synthetically generated by solving trajectory optimization problems for the cost function $c(\traj) = \sum_{r \in R}\sum_{t=1}^{T-1} \gamma_r (r_{t+1} - r_t)^2$, where $R=\{x,y,z,\dot\alpha,\dot\beta,\dot\gamma\}$ and $\gamma_r = 1$. Our algorithm assumes parametric cost uncertainty of $\gamma_r \in [0.01, 3]$, and we assume the cost function is known exactly for the baseline \cite{corl}. This problem is especially challenging for the baseline, since having many unknown constraint parameters can lead to non-identifiability of the constraint from the sampled trajectories. Furthermore, the nonlinearity of the quadrotor dynamics (we use the dynamics in \cite{corl}) makes sampling difficult. We compute $\guarsafe$/$\guarunsafe$ via Problems \ref{prob:query}/\ref{prob:kkt_vol}, taking 26.6/43.0 seconds on average over 10 different queries. Fig. \ref{fig:quad_stats} compares the coverage and accuracy of $\guarsafe$ and $\guarunsafe$ between our approach and the baseline \cite{corl} for each of the constraint spaces (position, pose, velocity, and control). The baseline and our approach perform comparably for some of the ``simpler" convex constraints (e.g. the angular velocity/control safe sets). However, the baseline struggles to learn the unsafe sets (due to the simultaneous identification of so many constraints from poor trajectory samples) and position constraints (as the quadrotor has second order dynamics, it is difficult to sample combinations of trajectories which uniquely imply a single state is unsafe). We also note that the baseline accuracies fluctuate greatly due to imperfect trajectory sampling and the difficulty of distinguishing between multiple constraints: different data may cause the optimization to switch unsafeness assignments from one constraint to another active constraint). By extracting volumes of $\guarsafe$ using the method in Sec. \ref{sec:method_volume}, we pass $\guarsafe$ to a trajectory optimizer \cite{casadi} to generate novel safe trajectories (Fig. \ref{fig:quad_demos}, right). This experiment suggests that by avoiding trajectory sampling, our KKT-based approach performs better on high-dimensional nonlinear systems.

\vspace{-6pt}
\section{Discussion and Conclusion}
\vspace{-5pt}

We present an algorithm which uses the KKT optimality conditions to determine constraints that make observed demonstrations appear locally-optimal with respect to an uncertain cost function. As the KKT conditions are an \textit{implicit} condition on the set of constraints that can possibly explain the demonstrations, we sidestep the shortcomings of previous methods \cite{wafr, corl} which rely on sampling lower-cost trajectories as explicit certificates of unsafeness. In future work, we aim to address two shortcomings of our method: first, we require the dynamics to be known in closed form, while \cite{wafr, corl} just need a simulator; second, the number of decision variables in our method scales linearly with the number of demonstrations, making it important that the demonstrations are informative with respect to the unknown constraint. To address these issues, we plan to extend our method to handle uncertain dynamics and develop an active learning method to obtain informative demonstrations.

\addtolength{\textheight}{-12cm}   



%
\vspace{-5pt}
\bibliographystyle{IEEEtran}
\bibliography{refs}

\end{document}